\documentclass[conference]{IEEEtran}
\IEEEoverridecommandlockouts
\usepackage{cite}
\usepackage{amsmath,amssymb,amsfonts}
\usepackage{amsthm}
\usepackage{algorithm}
\usepackage[mathlines,switch]{lineno}
\usepackage{algorithmic}
\usepackage{booktabs}
\usepackage{multirow}
\usepackage{graphicx}
\usepackage[caption=false,font=footnotesize]{subfig}
\usepackage{textcomp}
\usepackage{xcolor}
\usepackage{url}
\newtheorem{theorem}{Theorem}
\newtheorem{lemma}{Lemma}
\newtheorem{proposition}{Proposition}
\theoremstyle{definition}
\newtheorem{definition}{Definition}
\theoremstyle{remark}
\newtheorem{remark}{Remark}
\theoremstyle{plain}
\def\BibTeX{{\rm B\kern-.05em{\sc i\kern-.025em b}\kern-.08em
    T\kern-.1667em\lower.7ex\hbox{E}\kern-.125emX}}
\begin{document}

\title{Redistribution-based Cost Inference Improves Sparse
Safe Offline RL}

\author{
\IEEEauthorblockN{Ebenezer Gelo,
Geraud Nangue Tasse,
Steven James,
Benjamin Rosman}
\IEEEauthorblockA{
\textit{University of the Witwatersrand}\\
Johannesburg, South Africa\\
ebenezer.gelo1@students.wits.ac.za}
}

\maketitle

\begin{abstract}
Safe offline RL typically assumes access to dense per-step cost annotations, but in practice supervisors provide only trajectory-level stop-feedback: a binary signal at the first unsafe transition, with no per-step attribution. We frame this as a temporal credit assignment problem and propose the Redistribution-based Cost Inference (RCI) framework, which converts sparse stop-feedback into dense per-step costs via return decomposition, then trains a constrained offline policy on the augmented dataset. We show that return-equivalent redistribution preserves the feasible policy set and the optimal Lagrangian in a CMDP, establishing that the transformation is lossless in theory while yielding better-conditioned cost critic learning in practice. Experiments on highway driving and robotic manipulation demonstrate substantially lower violation rates than sparse and classifier-based baselines, with robustness to heterogeneous dataset compositions and label noise.

\end{abstract}


\section{Introduction}
\label{sec:introduction}

Safe reinforcement learning (RL) is central to deploying autonomous agents in high-stakes domains such as autonomous driving, clinical decision support, and robotic manipulation~\cite{Achiam_2017,review_safeRL}. The predominant formalism, the Constrained Markov Decision Process (CMDP)~\cite{Altman_1998}, encodes safety as a constraint on expected cumulative cost, but standard CMDP methods presuppose access to dense, per-step cost annotations---a requirement that is prohibitively expensive to satisfy in practice.

The offline RL paradigm exacerbates this gap. Agents trained on fixed historical datasets avoid the risks of online exploration~\cite{Fujimoto_2019,Levine_2020,Fujimoto_2021}, yet they inherit the safety profile of the behavior policy and face unreliable value and cost estimates under distributional shift~\cite{Kumar_2020}. Critically, offline datasets collected without safety instrumentation lack the cost labels required by existing constrained methods, and post-hoc expert annotation at the per-step level is intractable at scale.

A more realistic supervision signal is \emph{trajectory-level stop feedback}: a binary label indicating whether an episode was halted by a safety monitor upon detecting a violation. Such signals arise naturally from human oversight and automated safety systems, which operate by terminating execution rather than scoring individual transitions. However, each unsafe trajectory contributes only a single cost label at its termination point, providing no direct information about which earlier actions precipitated the violation. Recovering per-step costs from this signal is a temporal credit-assignment problem~\cite{Arjona-Medina_2019, Zhang_2023, Low_2025}, compounded here by distributional shift and the one-shot nature of offline data collection.

We address this problem with the \textbf{Redistribution-based Cost Inference (RCI)} framework, which converts sparse trajectory-level stop labels into dense per-step cost estimates via return decomposition, then trains safe policies using standard constrained offline RL. RCI is explicitly modular: the redistribution stage accepts any return-equivalent decomposition method (instantiated here with RUDDER~\cite{Arjona-Medina_2019}, though GRD~\cite{Zhang_2023} applies directly), and the policy learning stage accepts any constrained offline RL algorithm (instantiated with BCQ-Lagrangian~\cite{Fujimoto_2019}, though CPQ or CDT substitute without modification). We evaluate on highway driving and simulated robot control tasks across datasets generated by unsafe, random, and mixed behavior policies, demonstrating that RCI substantially reduces constraint violation rates while preserving task performance comparable to an unconstrained baseline.

\section{Preliminaries}

We model a task as a Markov decision process (MDP), defined by the tuple $(\mathcal{S}, \mathcal{A}, P, r, \gamma)$, where $\mathcal{S}$ is the state space, $\mathcal{A}$ is the action space, $P(s'|s,a)$ is the transition dynamics, $r(s,a)$ is the reward function, and $\gamma \in [0,1)$ is the discount factor. A policy $\pi(a|s)$ generates trajectories $\tau = (s_0, a_0,r_0, s_1, \dots, s_T)$ with cumulative return $R(\tau) = \sum_{t=0}^{T} \gamma^t r(s_t, a_t)$, and the objective is to find a policy that maximizes $\mathbb{E}[R(\tau)]$~\cite{sutton1998reinforcement}.

In \textit{offline} RL, the agent learns from a fixed dataset $\mathcal{D} = \{\tau_i\}_{i=1}^N$ of trajectories collected by a behavior policy $\mu$, without any further environment interaction. This avoids unsafe online exploration but introduces \emph{distributional shift}: the learned policy $\pi$ may select actions outside the support of $\mu$, rendering value estimates unreliable---a phenomenon known as extrapolation error. Offline RL methods address this by constraining the learned policy to remain close to the dataset distribution or by penalising out-of-support actions~\cite{Fujimoto_2019, Kumar_2020, Wu_2019}.

\subsection{Constrained Markov Decision Processes}
A constrained Markov decision process (CMDP) extends the MDP framework by introducing a cost function and a safety budget. Formally, a CMDP is given by $(\mathcal{S}, \mathcal{A}, P, r, c, \gamma, d)$, where $c(s,a)$ is a cost function representing constraint violations and $d$ is an allowable threshold. The agent’s objective is to maximize $\mathbb{E}[R(\tau)]$ subject to $\mathbb{E}[C(\tau)] \le d$, where $C(\tau) = \sum_{t=0}^T \gamma^t c(s_t, a_t)$ \cite{Altman_1998}. This formulation balances performance and safety; for example, an autonomous vehicle should minimize accidents while still reaching its destination efficiently. A common approach to solving CMDPs is Lagrangian relaxation, where the problem is converted to maximizing $\mathbb{E}[R(\tau) - \lambda C(\tau)]$ with the Lagrange $\lambda$ adjusted until the cost constraint is satisfied \cite{Achiam_2017, review_safeRL}.

\section{Related Work}
\paragraph{Offline Safe Reinforcement Learning.}
Offline safe RL combines the CMDP formalism with offline training on fixed 
datasets \cite{Polosky_2022, Xu_2022}. Conservative offline RL methods such 
as BCQ \cite{Fujimoto_2019} address distributional shift by constraining the 
learned policy to remain close to the behavior policy. Constrained extensions 
such as COPO \cite{Polosky_2022} and CPQ \cite{Xu_2022} incorporate 
Lagrangian penalties or conservative cost critics into this framework. A key 
limitation in previous studies is the assumption that per-step cost 
annotations are available in the dataset, an assumption we relax.

\paragraph{Cost Inference from Sparse Feedback.}
In practice, safety supervision is often sparse. \cite{Low_2025} propose 
TraCeS, which learns a dense safety scoring model from sparse trajectory 
labels using a multiplicative decomposition, but requires online labeler 
queries during training for iterative refinement. Similarly, \cite{Chirra_2024} 
introduce RLSF, which transforms segment-level feedback into classification 
targets, but also relies on interactive labeler access during online learning. 
These settings differ fundamentally from ours: the dataset is fixed, the 
labeler cannot be queried post-hoc, and all safety supervision must be 
incorporated offline. 

\paragraph{Return Decomposition and Credit Assignment.}
Return decomposition methods address credit assignment in sparse reward RL 
by training sequence models to predict cumulative returns and redistributing 
terminal signals backward through time \cite{Arjona-Medina_2019, Zhang_2023}. 
These methods guarantee return equivalence: the redistributed per-step signals 
sum to the original trajectory return. Our work transfers this principle to 
the cost and constraint domain, establishing that return-equivalent cost 
redistribution preserves the feasible policy set and the optimal Lagrangian 
in a CMDP.

\section{Redistribution-based Cost Inference}
\label{sec:methods}

We consider an offline dataset $\mathcal{D} = \{\tau_i\}_{i=1}^{N}$ of trajectories where each trajectory $\tau = \{(s_0,a_0,r_0), \ldots, (s_T,a_T,r_T)\}$ is labeled by an expert function $\mathcal{F}(\tau)$ returning the index of the first unsafe transition, or the empty set if the trajectory is safe. This produces a sparse cost signal:
\begin{equation}
c^{\text{sparse}}(s_t,a_t) =
\begin{cases}
1, & \text{if } t \in \mathcal{F}(\tau); \\
0, & \text{otherwise}.
\end{cases}
\end{equation}
The objective is to learn a policy $\pi$ maximizing expected return $\mathbb{E}_{\pi}[R(\tau)]$ subject to $\mathbb{E}_{\pi}[C(\tau)] \leq d$ for a specified budget $d$, using only $\mathcal{D}$ and sparse cost annotations, within the CMDP framework with soft constraints~\cite{Altman_1998, Garcia_2015}. The central challenge is that stop-feedback provides trajectory-level information---indicating that something went wrong---but not dense per-step annotations of which actions caused the failure, creating a severe credit assignment problem~\cite{Arjona-Medina_2019}.


We propose Redistribution-based Cost Inference (RCI), a modular framework that decomposes safe offline RL into three stages: (1)~trajectory-level stop-feedback collection, (2)~return-decomposition-based cost inference, and (3)~constrained offline policy learning. The key insight is that credit assignment for costs and policy learning for safety are algorithmically distinct problems that benefit from separate treatment. This modularity allows any return-equivalent decomposition method and any constrained offline RL algorithm to be substituted independently.

\paragraph{Stage 1: Expert Feedback.}
For each trajectory $\tau_i$, an expert provides binary feedback: $C(\tau_i) = 1$ if the trajectory contains an unsafe transition, and $C(\tau_i) = 0$ otherwise. Unsafe trajectories are truncated at the first violation point $t^*$. This labeling protocol reflects how safety monitoring operates in practice---automated systems or human supervisors issue stop commands at the first detected violation~\cite{Poletti_2023}---and is far more scalable than dense per-step annotation.

\paragraph{Stage 2: Return Decomposition.}
We train a sequence model $\hat{C}(s_{0:t}, a_{0:t})$ to predict episodic cost $C(\tau)$ from trajectory prefixes by minimizing $\mathcal{L} = \mathbb{E}_{\tau \sim \mathcal{D}}[(\hat{C}(s_{0:T}, a_{0:T}) - C(\tau))^2]$. Dense per-step costs are then obtained by differencing successive predictions:
\begin{equation}
\tilde{c}_t = \hat{C}(s_{0:t},a_{0:t}) - \hat{C}(s_{0:t-1},a_{0:t-1}) + \delta_t, \quad \hat{C}(s_{0:-1},a_{0:-1}) = 0,
\label{eq:cost_redistribution}
\end{equation}
where $\delta_T = C(\tau) - \hat{C}(s_{0:T},a_{0:T})$ at the terminal step (with $\delta_t = 0$ for $t < T$) is a compensation term ensuring exact return equivalence (Definition~\ref{def:return_equiv}). The intuition is straightforward: if the model's cost prediction increases sharply at timestep $t$, that transition contains information critical to predicting eventual failure. Following prior work~\cite{Arjona-Medina_2019}, we instantiate $\hat{C}$ as an LSTM, though any sequence architecture (e.g., Transformers) would suffice.


\paragraph{Theoretical Results.}
A redistributed cost $\tilde{c}$ is return-equivalent to the sparse cost if $\sum_{t=0}^{T} \tilde{c}_t = C(\tau)$ for every trajectory (Definition~\ref{def:return_equiv}). The costs in Equation~\eqref{eq:cost_redistribution} satisfy this by construction via a telescoping sum (Lemma~\ref{lem:return_equiv}). This property cascades through the constrained optimization (formal statements and proofs are provided in Appendix~\ref{app:theoretical_res}):

\begin{proposition}[Constraint Equivalence]
\label{prop:constraint_equiv}
If $\tilde{c}$ is return-equivalent to $c^{\emph{sparse}}$, then for any policy $\pi$:
\begin{equation}
\mathbb{E}_{\tau \sim \pi}\left[\sum_{t=0}^{T} \tilde{c}_t\right] = \mathbb{E}_{\tau \sim \pi}\left[C(\tau)\right].
\end{equation}
\end{proposition}

\begin{theorem}[Policy Invariance]
\label{thm:policy_invariance}
Let $\mathcal{M}^{\emph{sparse}}$ and $\mathcal{M}^{\emph{redist}}$ denote CMDPs differing only in instantaneous costs. Then the feasible policy sets are identical and $\pi^*_{\emph{sparse}} = \pi^*_{\emph{redist}}$. Moreover, for any $\lambda \geq 0$:
\begin{equation}
\mathcal{L}(\pi,\lambda;\tilde{c}) = \mathcal{L}(\pi,\lambda;c^{\emph{sparse}}),
\end{equation}
so the Lagrangian saddle point $(\pi^*,\lambda^*)$ coincides under both formulations.
\end{theorem}

Crucially, these guarantees hold regardless of the sequence model's prediction accuracy due to the compensation term $\delta_T$. While both formulations define the same optimal policy, redistributed costs provide dense supervision throughout trajectories, yielding a better-conditioned regression target for the cost critic compared to sparse terminal signals (Remark~\ref{remark_2}).

\paragraph{Stage 3: Policy Optimization.}
Given the dense cost-augmented dataset $\mathcal{D}_{\text{dense}} = \{(s_t, a_t, r_t, \tilde{c}_t)\}$, we train a constrained offline RL policy. In our instantiation, we employ BCQ-Lagrangian~\cite{Fujimoto_2019, Liu_2023}, which addresses distributional shift through a VAE-based behavior cloning constraint and enforces safety via an adaptive Lagrange multiplier. The Bellman target becomes $y = r(s,a) + \gamma \max_{a'} [Q_{\theta'}(s',a') - \lambda\,\tilde{c}(s',a')]$, where $\lambda$ is updated as $\lambda \leftarrow \max(0, \lambda + \alpha_\lambda(C_{\text{batch}} - d))$ based on batch-average cost relative to the budget $d$. This choice is well-motivated but not essential to RCI.

\section{Experiments}
\label{exp_setup}

We evaluate RCI on two benchmark environments: \texttt{HighwayEnv}~\cite{highwayenv}, where an ego vehicle navigates highway traffic while avoiding collisions, and \texttt{Safe-FetchReach}~\cite{robotics}, where a 7-DOF robotic arm reaches target positions while avoiding a spherical hazard region. For each environment we generate 5{,}000 offline episodes using PPO-trained behavioral policies that optimize task reward while disregarding safety~\cite{Schulman_2017}, producing aggressive driving with frequent collisions and hazard-ignoring reaching behavior. An automated evaluator identifies the first unsafe transition in each trajectory and truncates all subsequent steps, mimicking expert intervention~\cite{Poletti_2023}: a transition is labeled unsafe when the ego approaches within 0.2 units of another vehicle in \texttt{HighwayEnv}, or when the end-effector enters the spherical hazard region $\mathcal{H} = \{p \in \mathbb{R}^3 : \|p - h\|_2 \leq r\}$ in \texttt{Safe-FetchReach}. 

\textbf{Baselines and Protocol.} We compare against three baselines under identical BCQ-Lagrangian architectures~\cite{Fujimoto_2019}: (1) \textit{Reward-Only}, which ignores safety costs and sets the budget to infinity; (2) \textit{Sparse}, which uses the raw terminal cost label without redistribution; and (3) \textit{Hazard}, a two-head binary classifier $\tilde{c}(s_t, a_t) = P_1(s_t, a_t) + P_2(s_t, a_t)$, where $P_1$ predicts whether a state-action pair appears in any unsafe trajectory and $P_2$ predicts the unsafe termination point specifically, trained with focal loss to address class imbalance~\cite{Lin2017}. For RCI we instantiate RUDDER-based redistribution with an LSTM sequence model~\cite{Arjona-Medina_2019}. For each method we sweep the safety budget $d$ over the 10th-50th percentiles of the dataset's trajectory cost distribution in increments of 10, train three independent policies per configuration, and evaluate the lowest-violation policy over 1{,}000 online episodes using ground-truth safety labels. 

\subsection{Safe Navigation}

Figure~\ref{fig:highway_results} reports normalized return $R_{\text{norm}} = (R - R_{\min}) / (R_{\max} - R_{\min})$ pooled across all evaluated policies and seeds just as in ~\cite{Liu_2023}, and violation rate on \texttt{HighwayEnv} under the mixed-composition dataset (equal-proportioned PPO and uniform random rollouts). RCI cuts violation rate compared to Sparse and Hazard while maintaining task return; a two-sample t-test against the unconstrained baseline yields $t = 0.9962$, $p = 0.3483$, indicating no statistically significant return penalty.

\begin{figure}[!h]
    \centering
    \subfloat[\texttt{HighwayEnv} results\label{fig:highway_results_core}]{%
        \includegraphics[width=0.48\linewidth]{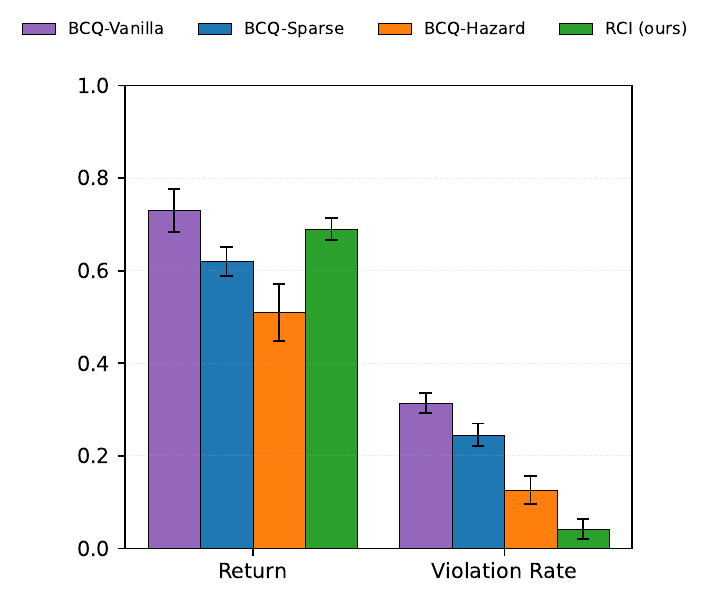}}
    \subfloat[Spatial cost landscapes\label{fig:ccval_heatmap}]{%
        \includegraphics[width=0.48\linewidth]{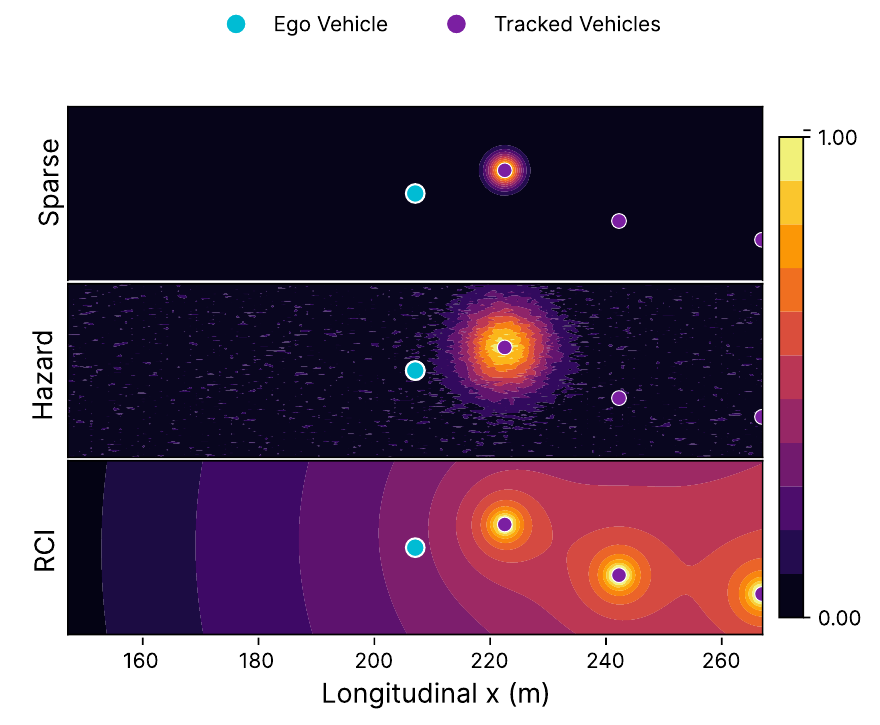}}
    \caption{Performance on \texttt{HighwayEnv}. (a) Normalized return and violation rate across baselines and RCI; bars show means over 5 seeds with standard error. The two-sample t-test between RCI and BCQ-Vanilla yields $t = 0.9962$, $p = 0.3483$. (b) Learned cost critics as heatmaps over road geometry for Sparse, Hazard, and RCI (top to bottom); warmer colors indicate higher predicted cumulative cost.}
    \label{fig:highway_results}
\end{figure}

\textbf{Learned spatial risk structure.} Figure~\ref{fig:highway_results}(b) shows learned cost critics evaluated over the \texttt{HighwayEnv} road geometry. Sparse concentrates all cost at the terminal unsafe configuration, with near-zero values elsewhere, even in states already dangerously close to other vehicles. Hazard's classifier captures near-terminal danger in a small neighborhood but decays rapidly. RCI produces globally coherent spatial structure: cost increases gradually as the ego approaches traffic, with high-cost regions extending backward along approach corridors. The redistributed costs treat vehicle proximity as a graded risk signal modulated by temporal context, confirming that the dense supervision recovers meaningful temporal causality from a single terminal bit.

We report two additional experiments on \texttt{HighwayEnv} that probe RCI's behavior under realistic deployment conditions: varying data sources and imperfect supervision.

\textbf{Dataset Composition}
%
We evaluate RCI across three behavioral policies with distinct exploration patterns: (i) \textit{PPO}, task-optimized policies producing goal-directed rollouts concentrated around high-reward regions; (ii) \textit{Random}, uniform random action sampling; and (iii) \textit{Mixed}, equal-proportion combinations of PPO and Random rollouts, reflecting realistic offline datasets that aggregate multiple data sources.

\begin{figure}[!h]
    \centering
    \subfloat[Normalized return\label{fig:dc_return}]{%
        \includegraphics[width=0.48\linewidth]{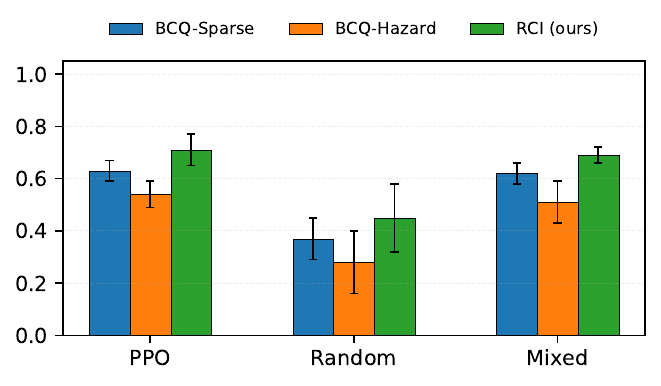}}
    \hfill
    \subfloat[Violation rate\label{fig:dc_violation}]{%
        \includegraphics[width=0.48\linewidth]{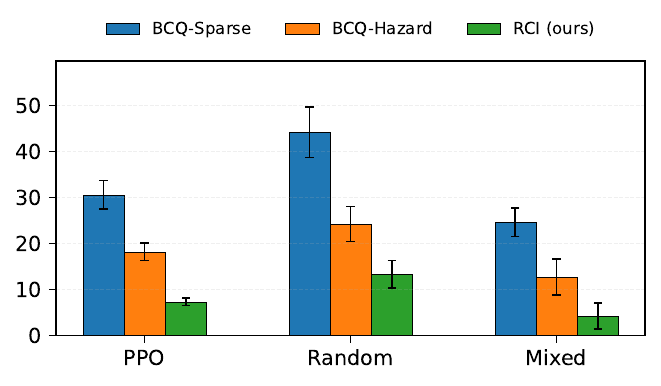}}
    \caption{Dataset composition ablation on \texttt{HighwayEnv}. Bars show means over 3 seeds with standard error; each policy is evaluated for 1{,}000 episodes per seed using ground-truth violations. Safety budget and model selection rule are held fixed across methods.}
    \label{fig:data_composition}
\end{figure}

RCI's violation reduction holds across all three regimes (Figure~\ref{fig:data_composition}), indicating that the redistribution mechanism does not depend on the behavior policy producing structured or near-optimal exploration.

\textbf{Label Noise}
%
We simulate two corruption regimes. \textit{Noisy labels}: for each unsafe trajectory, the stop-point index is shifted by $\delta$ steps, where $\delta$ is uniformly drawn from $[-15, 15]$ in increments of 5, subject to trajectory bounds. This perturbs the termination label earlier (false positive) or later (false negative) relative to the true unsafe transition, simulating limited-precision human or automated annotation. \textit{Adversarial labels}: safety labels of a random 20\% of trajectories are flipped: unsafe trajectories are relabeled safe (label removed), and safe trajectories are relabeled unsafe with a stop label placed at a uniformly random timestep. Cost budget $d$ is fixed to the value selected under clean supervision to isolate the effect of corruption.

\begin{figure}[!h]
    \centering
    \subfloat[Normalized return\label{fig:ln_return}]{%
        \includegraphics[width=0.48\linewidth]{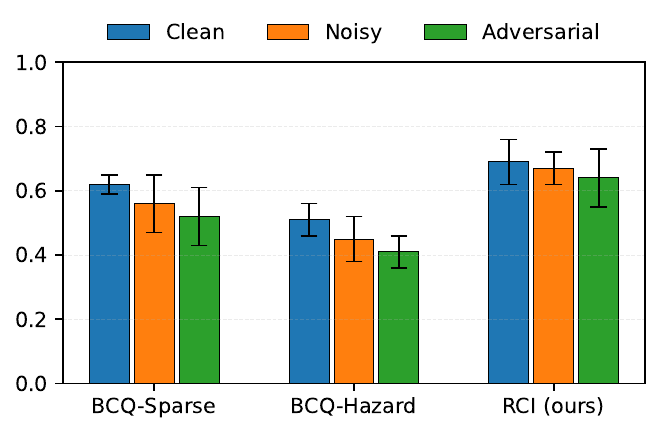}}
    \hfill
    \subfloat[Violation rate\label{fig:ln_violation}]{%
        \includegraphics[width=0.48\linewidth]{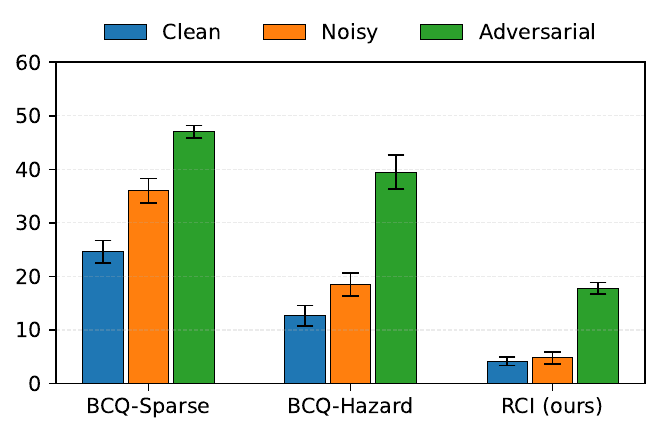}}
    \caption{Label-noise ablation on \texttt{HighwayEnv}, with cost budget $d$ fixed to the value selected in the main results (Figure~\ref{fig:highway_results}). Error bars show standard error across 3 seeds; evaluation uses 1{,}000 episodes and ground-truth safety events.}
    \label{fig:label_noise}
\end{figure}

Returns remain stable across noise regimes (Figure~\ref{fig:ln_return}) since termination-time noise alters constraint signals without affecting logged rewards. Violation rates on Sparse and Hazard respond more sharply to misaligned supervision than RCI (Figure~\ref{fig:ln_violation}), consistent with the smoothing effect of return decomposition over per-step label errors.

\subsection{Safe Continuous Control} 
On \texttt{Safe-FetchReach} (Figure~\ref{fig:reacher_results}), we observe the same trend: RCI achieves competitive returns while substantially reducing violation rates. The vector field shows strong repulsive behavior away from the hazard region under a restrictive budget, demonstrating that spatially coherent avoidance emerges from trajectory-level stop-feedback alone, without access to dense cost annotations.

\begin{figure}[htb]
    \centering
    \subfloat[\texttt{Safe-FetchReach} results\label{fig:reacher_results_core}]{%
        \includegraphics[width=0.48\linewidth]{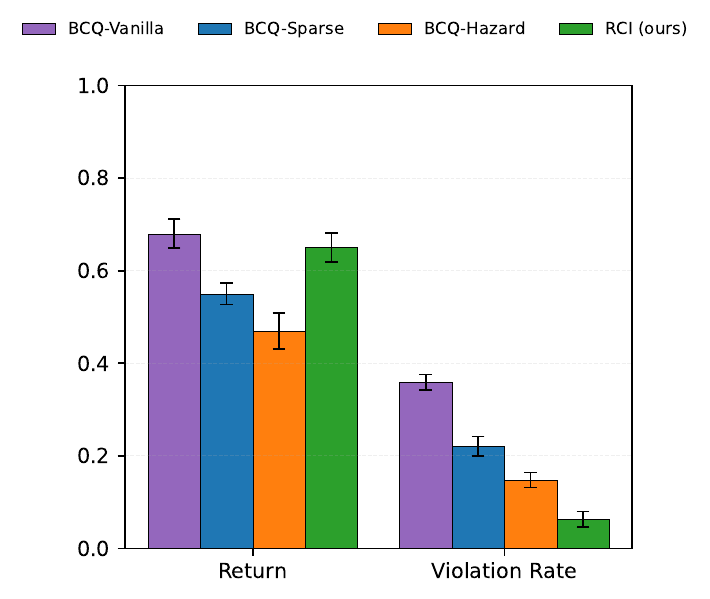}}
    \subfloat[Policy vector field\label{fig:vf_results}]{%
        \includegraphics[width=0.42\linewidth]{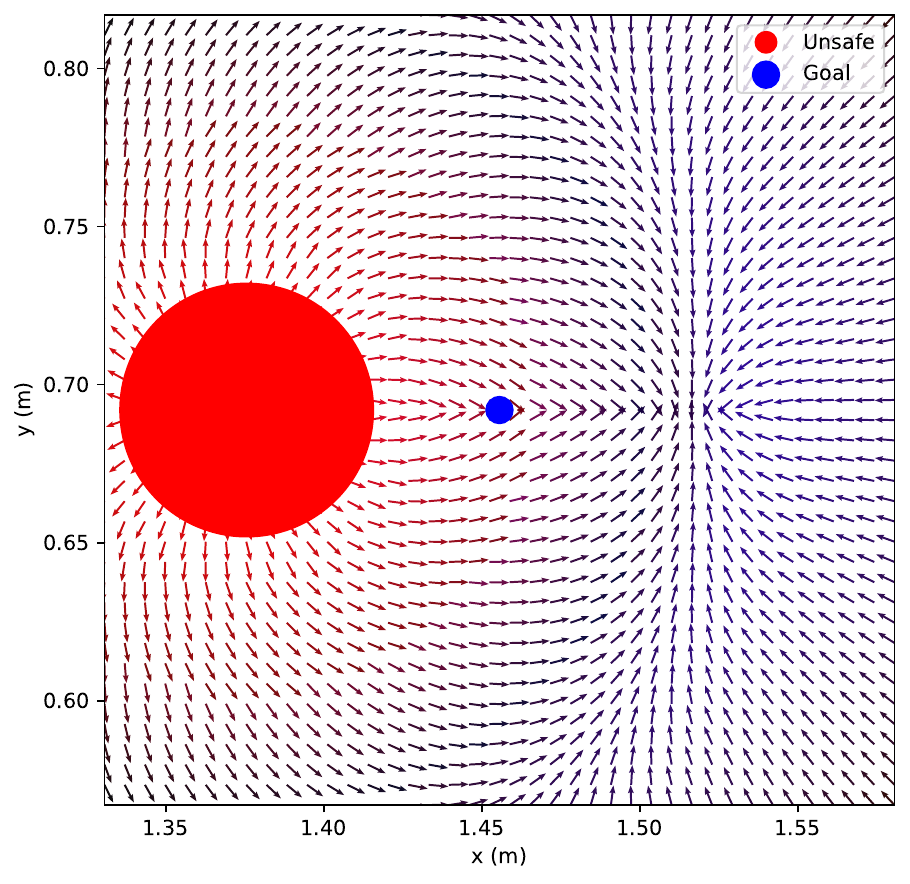}}
    \caption{Performance on \texttt{Safe-FetchReach}. (a) Normalized return and violation rate; bars show means over 5 seeds with standard error. (b) Vector field of RCI policy under a restrictive safety budget; arrows show action directions, color encodes critic value estimates.}
    \label{fig:reacher_results}
\end{figure}

\section{Conclusion}
\label{sec:conclusion}

RCI converts trajectory-level stop-feedback into dense per-step costs via return decomposition, enabling constrained offline policy learning from the supervision modality physical AI deployments actually produce. The transformation preserves the feasible policy set and Lagrangian saddle point of the underlying CMDP, and reduces violation rates roughly fivefold on two physical-safety domains without statistically significant return penalties. Stop-feedback is informationally sufficient for safe offline policy learning; richer supervision modalities are not prerequisites.

Three limitations matter. RCI inherits standard offline RL coverage requirements: datasets skewed toward unsafe trajectories risk overly conservative policies, and sparse safe coverage may leave policy optimization underspecified. The redistribution mechanism captures statistical, not causal, associations between trajectory prefixes and violations, identifying risky situations rather than causally hazardous actions. The framework is currently restricted to single binary constraints; multi-constraint extension via per-channel decomposition is straightforward but raises open questions about balancing competing objectives. Future directions include graduated severity feedback, uncertainty-aware redistribution under distributional shift, and more expressive sequence architectures.

\bibliographystyle{IEEEtran}
\bibliography{ref}

\appendices

\section{Signal Preservation}
\label{app:theoretical_res}
Algorithm~\ref{alg:rci} summarizes the complete RCI framework. The framework takes as input an offline dataset $\mathcal{D}$, a labeler $\mathcal{L}$ (human or automated), a return decomposition algorithm $\mathcal{A}_{\text{decomp}}$, a constrained offline RL algorithm $\mathcal{A}_{\text{OSRL}}$, and a cost budget $d$. It outputs a safe policy $\pi$.

\begin{algorithm}[htb]
\caption{RCI Framework}
\label{alg:rci}
\begin{algorithmic}[1]
\STATE \textbf{Input:} Dataset $\mathcal{D} = \{\tau_i\}_{i=1}^N$, Labeler $\mathcal{L}$, Return decomposition algorithm $\mathcal{A}_{\text{decomp}}$, Constrained Offline RL algorithm $\mathcal{A}_{\text{OSRL}}$, Cost budget $d$
\STATE \textbf{Output:} Safe policy $\pi$
\FOR{each trajectory $\tau_i \in \mathcal{D}$} 
    \STATE $F(\tau_i) \leftarrow \mathcal{L}(\tau_i)$ \hfill $\triangleright$ Feedback Collection
    \STATE $C(\tau_i) \leftarrow \begin{cases} 
        1 & \text{if } F(\tau_i) \neq \emptyset \\
        0 & \text{otherwise}
    \end{cases}$
    \FOR{$t = 0, 1, \ldots, T_i$}
        \STATE $c^{\text{sparse}}(s_t, a_t) \leftarrow \begin{cases}
            1 & \text{if } t \in F(\tau_i) \\
            0 & \text{otherwise}
        \end{cases}$
    \ENDFOR
\ENDFOR
\STATE $\hat{C} \leftarrow$ Train sequence model with $\mathcal{A}_{\text{decomp}}$ on $\{(\tau_i, C(\tau_i))\}$ \hfill $\triangleright$ Return Decomposition
\FOR{each trajectory $\tau_i \in \mathcal{D}$}
    \FOR{$t = 0, 1, \ldots, T_i$}
        \STATE $\tilde{c}_t \leftarrow \hat{C}(s_{0:t},a_{0:t}) - \hat{C}(s_{0:t-1},a_{0:t-1})$ \hfill $\triangleright$ with $\hat{C}(s_{0:-1},a_{0:-1}) = 0$
    \ENDFOR
    \STATE $\delta_i \leftarrow C(\tau_i) - \hat{C}(s_{0:T_i}, a_{0:T_i})$ \hfill $\triangleright$ Return-equivalence correction
    \STATE $\tilde{c}_{T_i} \leftarrow \tilde{c}_{T_i} + \delta_i$
\ENDFOR
\STATE $\mathcal{D}_{\text{dense}} \leftarrow \{(s_t, a_t, r_t, \tilde{c}_t)\}$ \hfill $\triangleright$ Constrained Policy Learning
\STATE $\pi \leftarrow \mathcal{A}_{\text{OSRL}}(\mathcal{D}_{\text{dense}}, d)$
\RETURN $\pi$
\end{algorithmic}
\end{algorithm}

A critical question still arises: does this redistribution preserve the original supervision signal? In the offline setting, where no additional feedback can be queried, any transformation of the cost signal must guarantee that the resulting constrained optimization problem yields the same optimal policy. We now establish this formally.

\begin{definition}[Return-Equivalent Cost Redistribution]
\label{def:return_equiv}
A redistributed cost function $\tilde{c}: \mathcal{S} \times \mathcal{A} \to \mathbb{R}$ is return-equivalent to the sparse cost if for every trajectory $\tau$, the sum of redistributed costs equals the episodic cost: $\sum_{t=0}^{T} \tilde{c}_t = C(\tau)$.
\end{definition}

\begin{lemma}[Return Equivalence]
\label{lem:return_equiv}
The redistributed costs defined in Equation~\eqref{eq:cost_redistribution} satisfy Definition~\ref{def:return_equiv} for any sequence model $\hat{C}$.
\end{lemma}

\begin{proof}
Summing across all timesteps yields a telescoping series:
\begin{align}
\sum_{t=0}^{T} \tilde{c}_t &= \sum_{t=0}^{T} \left[ \hat{C}(\tau_{0:t}) - \hat{C}(\tau_{0:t-1}) \right] + \delta_T \\
&= \hat{C}(\tau_{0:T}) - \hat{C}(\tau_{0:-1}) + \delta_T \\
&= \hat{C}(\tau_{0:T}) + \left[ C(\tau) - \hat{C}(\tau_{0:T}) \right] = C(\tau).
\end{align}
The boundary condition $\hat{C}(\tau_{0:-1}) = 0$ and the compensation term $\delta_T$ ensure exact equality regardless of prediction accuracy.
\end{proof}

This return equivalence property is essential in the offline setting where we cannot query for additional supervision. Unlike interactive learning where the agent can request more safety labels \cite{Low_2025, Chirra_2024}, offline learning must extract maximum information from fixed labels. Return equivalence ensures that the inferred costs preserve the original supervision signal without introducing systematic bias, while redistributing it temporally to guide per-step decision-making.

We now show that return equivalence implies policy invariance in the constrained MDP.


\begin{proof}
By Lemma~\ref{lem:return_equiv}, return equivalence implies $\sum_{t=0}^{T} \tilde{c}_t = C(\tau)$ for every trajectory $\tau$. Since this equality holds pointwise, the result follows by linearity of expectation.
\end{proof}

\begin{proof}
By Proposition~\ref{prop:constraint_equiv}, $\mathbb{E}_{\tau \sim \pi}[\sum_{t=0}^{T} \tilde{c}_t] = \mathbb{E}_{\tau \sim \pi}[C(\tau)]$ for any $\pi$, so the feasible sets are identical. Both CMDPs share identical state spaces, action spaces, transition dynamics, reward functions, and discount factors, and maximize the same expected return over identical feasible sets; hence their optimal solutions coincide.
\end{proof}

Since constrained MDPs are commonly solved via Lagrangian relaxation \cite{Altman_1998, Achiam_2017}, we verify that policy invariance extends to this optimization framework.


\paragraph{Lagrangian Equivalence.}
\label{thm:lagrangian}
For any policy $\pi$ and multiplier $\lambda \geq 0$, the Lagrangian evaluated using sparse costs equals the Lagrangian evaluated using return-equivalent redistributed costs. Consequently, the saddle point $(\pi^*, \lambda^*)$ is identical for both formulations.

\begin{proof}
By Proposition~\ref{prop:constraint_equiv}:
\begin{align}
\mathcal{L}(\pi, \lambda; \tilde{c}) &= \mathbb{E}_{\tau \sim \pi}[R(\tau)] - \lambda \left( \mathbb{E}_{\tau \sim \pi}\left[ \sum_{t=0}^T \tilde{c}_t \right] - d \right) \\
&= \mathbb{E}_{\tau \sim \pi}[R(\tau)] - \lambda \left( \mathbb{E}_{\tau \sim \pi}[C(\tau)] - d \right) = \mathcal{L}(\pi, \lambda; c^{\text{sparse}}).
\end{align}
Since the Lagrangian is identical for all $(\pi, \lambda)$, the saddle point problem yields the same solution under both cost specifications.
\end{proof}

These results collectively establish that RCI preserves supervision at multiple levels: trajectory-level (Lemma~\ref{lem:return_equiv}), policy-level (Proposition~\ref{prop:constraint_equiv}), constraint-level, and optimization-level (Theorem~\ref{thm:policy_invariance}). The transformation changes only the temporal allocation of costs---from sparse terminal signals to dense per-step signals---while preserving the aggregate constraint information and the set of optimal policies.

\begin{remark}
The compensation term $\delta_T$ ensures that these guarantees hold regardless of the sequence model's prediction accuracy. While prediction errors affect the distribution of costs across timesteps and may degrade credit assignment quality, they do not compromise supervision preservation. Under perfect prediction, the compensation term vanishes and the redistributed cost exactly quantifies how much each transition increased the probability of eventual violation.
\end{remark}

\begin{remark}
\label{remark_2}
Theorem~\ref{thm:policy_invariance} establish that the Lagrangian and optimal policy are identical under sparse and redistributed costs. However, practical algorithms do not optimize the true Lagrangian directly—they estimate cost expectations using learned critics. Under sparse costs, the cost critic observes non-zero signal only at terminal failure points, making it difficult to learn accurate cost Q-values for earlier states. This is a supervised learning problem with extreme sparsity. Redistributed costs provide dense supervision throughout trajectories, yielding a better-conditioned regression target for the critic. Both formulations define the same optimal policy, but dense costs enable more accurate critic estimation within a fixed sample budget.
\end{remark}

\section{Qualitative Analysis}

This section provides qualitative evidence supporting the quantitative results presented in Figure~\ref{fig:highway_results} and Figure~\ref{fig:reacher_results}. We visualize spatial cost landscapes learned by each method, examine policy trajectories under different safety budgets, and present vector fields of learned policies in the manipulation domain.

\subsection{\texttt{HighwayEnv}}
A central question motivating our approach is whether return decomposition produces meaningful dense cost signals that identify safety-critical steps. We analyze the quality of redistributed costs through visualization and quantitative assessment.

 Figure~\ref{fig:cost_redistribution} presents the cost redistribution process on randomly sampled  trajectories from \texttt{HighwayEnv}. Figure~\ref{fig:cost_sparse_samples} shows the sparse trajectory-level label (a single cost of 1 at the terminal unsafe step), while the Figure~\ref{fig:cost_dense_samples} shows the dense per-step costs inferred by redistribution via return decomposition.
\begin{figure}[h!]
    \centering
    \subfloat[Sparse samples\label{fig:cost_sparse_samples}]{%
        \begin{minipage}{\linewidth}
            \centering
            \includegraphics[width=0.32\linewidth]{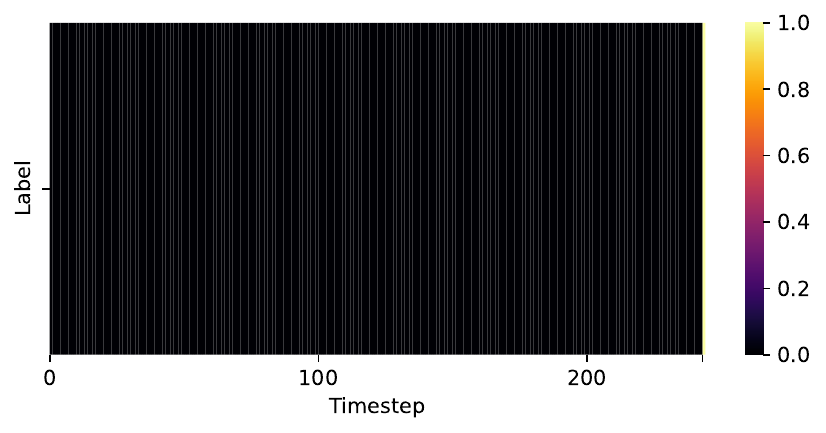}\hfill
            \includegraphics[width=0.32\linewidth]{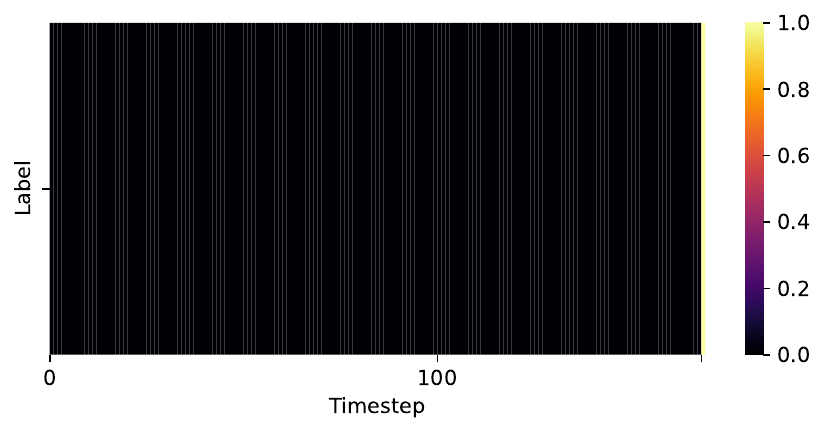}\hfill
            \includegraphics[width=0.32\linewidth]{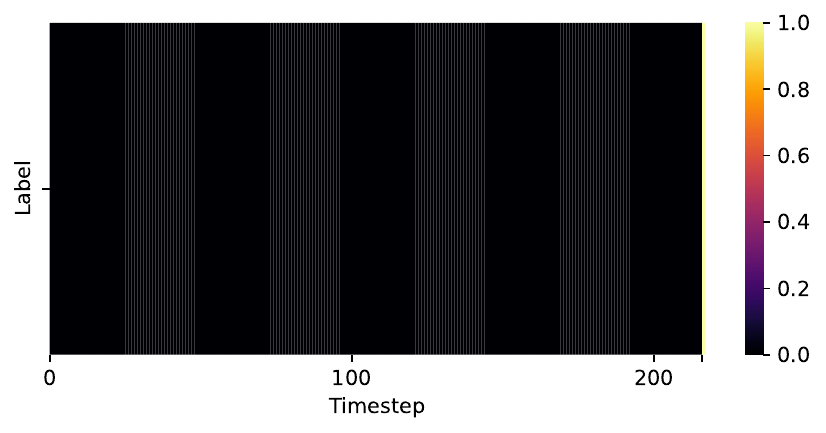}
        \end{minipage}}
    \vspace{0.7em}
    \subfloat[Dense samples\label{fig:cost_dense_samples}]{%
        \begin{minipage}{\linewidth}
            \centering
            \includegraphics[width=0.32\linewidth]{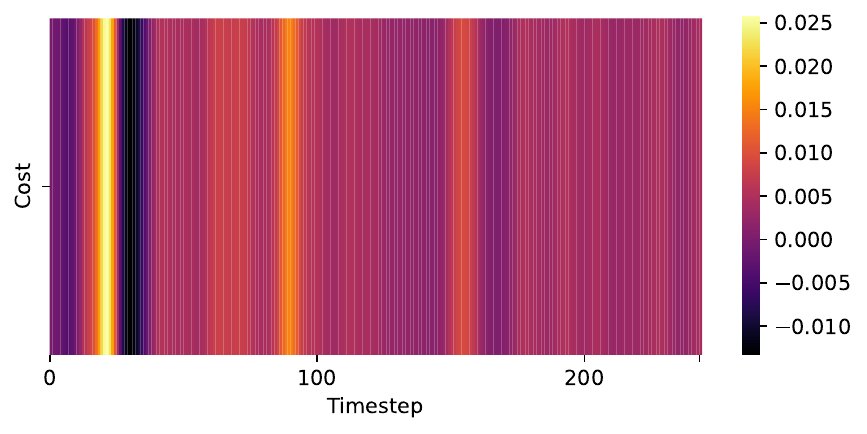}\hfill
            \includegraphics[width=0.32\linewidth]{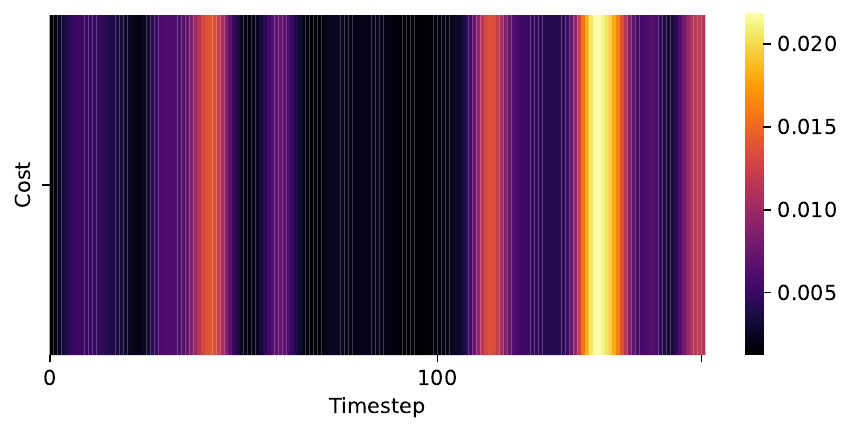}\hfill
            \includegraphics[width=0.32\linewidth]{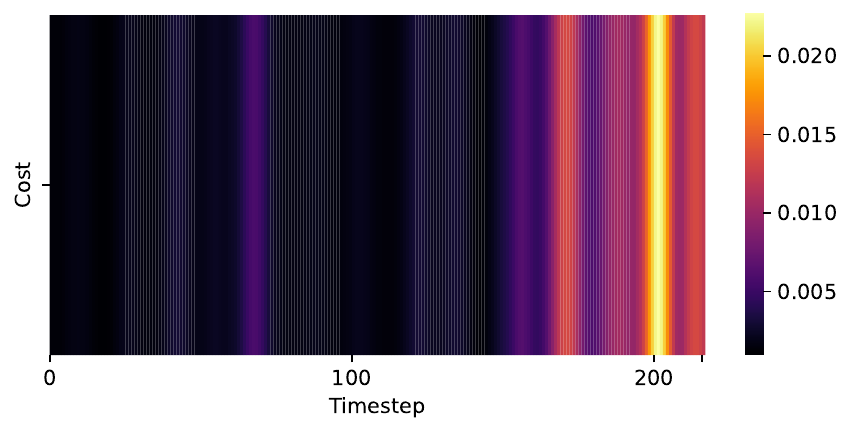}
        \end{minipage}}
    \vspace{0.7em}
    \subfloat[Distance to nearest vehicle\label{fig:distance_samples}]{%
        \begin{minipage}{\linewidth}
            \centering
            \includegraphics[width=0.32\linewidth]{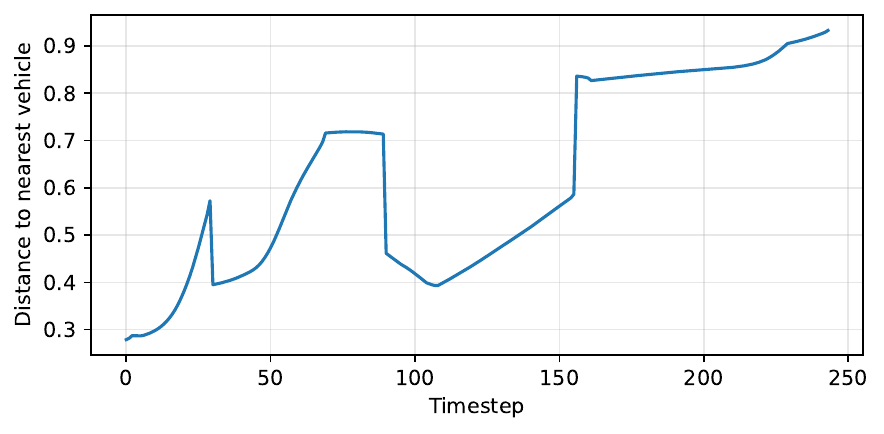}\hfill
            \includegraphics[width=0.32\linewidth]{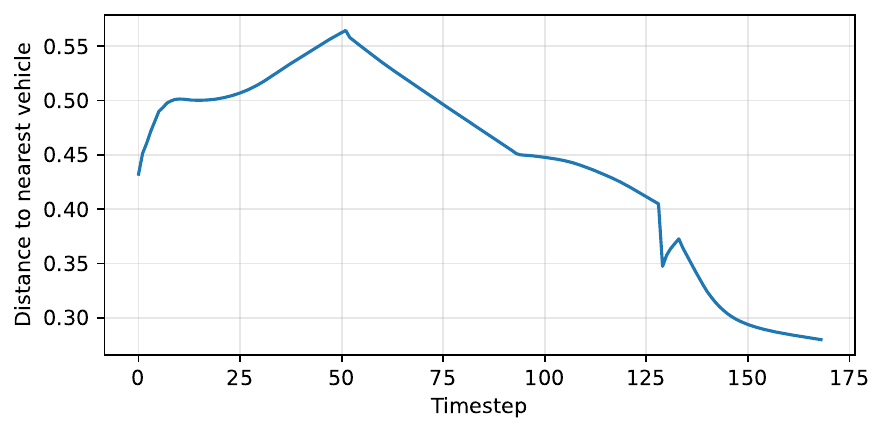}\hfill
            \includegraphics[width=0.32\linewidth]{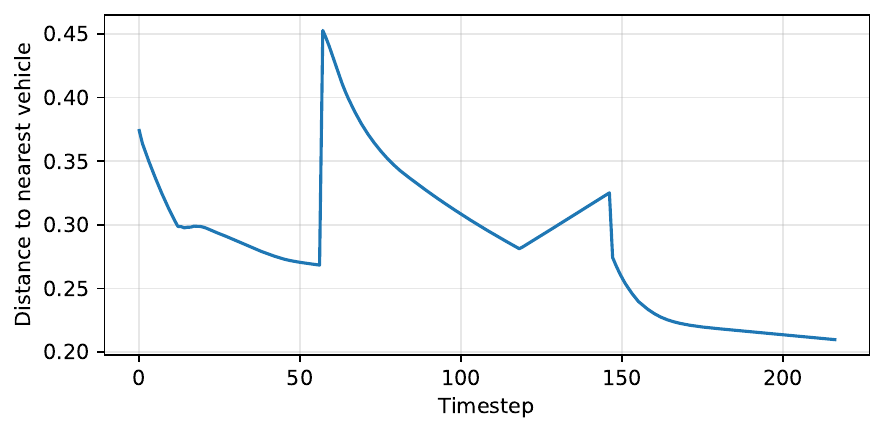}
        \end{minipage}}
    \caption{Cost inference on three sampled trajectories from \texttt{HighwayEnv}. (\ref{fig:cost_sparse_samples}) shows the sparse stop-feedback labels (a single terminal cost at the unsafe step). (\ref{fig:cost_dense_samples}) shows the dense per-step costs inferred via return decomposition, with higher values assigned to precursor actions that contributed to the eventual violation. (\ref{fig:distance_samples}) shows the distance to the nearest vehicle over time for the same trajectories; note that each trajectory is truncated at the first unsafe transition, as recorded in the offline dataset.}
    \label{fig:cost_redistribution}
\end{figure}

\paragraph{Policy Trajectories Under Different Safety Budgets.}
Under a restrictive budget ($d = P_{10}$), the policy maintains large following distances and exits the highway entirely; under a balanced budget ($d = P_{30}$), it sustains forward progress while respecting distance constraints. These rollouts confirm that redistributed costs yield interpretable safety--performance trade-offs without retraining the cost inference model.

\begin{figure}[h]
    \centering
    \subfloat[$d = P_{10}$\label{fig:restrictive_budget}]{%
        \includegraphics[width=0.45\linewidth]{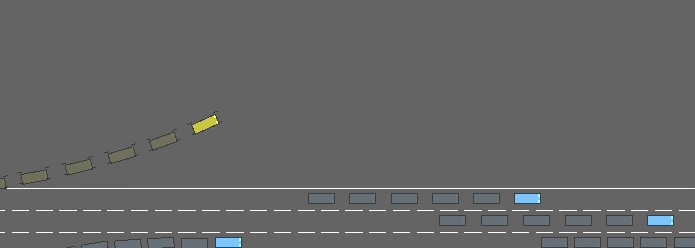}}
    \hfill
    \subfloat[$d = P_{30}$\label{fig:balanced_budget}]{%
        \includegraphics[width=0.45\linewidth]{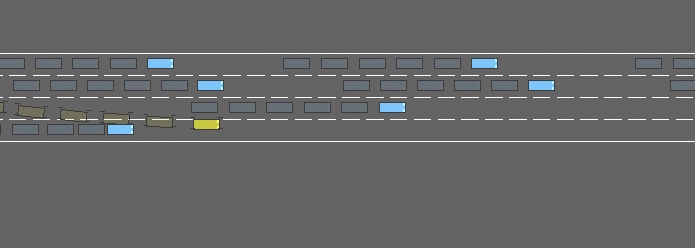}}
    \caption{Trajectories from RCI's learned policy in \texttt{HighwayEnv} under restrictive (\ref{fig:restrictive_budget}) and balanced (\ref{fig:balanced_budget}) safety budgets $d$, where $P_q$ denotes the $q^{th}$ percentile of the dataset's episodic cost distribution.}
    \label{fig:highway_budget}
\end{figure}

Figure~\ref{fig:reacher_vf} presents vector fields of the learned RCI policies in \texttt{Safe-FetchReach} across three safety budgets from the empirical percentile sweep. Arrows depict the policy's action directions in a two-dimensional projection of the state space, and color encodes the critic's value estimates.

\begin{figure}[h!]
    \centering
    \subfloat[\texttt{Safe-FetchReach} Workbench\label{fig:reacher_workbench}]{%
        \includegraphics[width=0.37\linewidth]{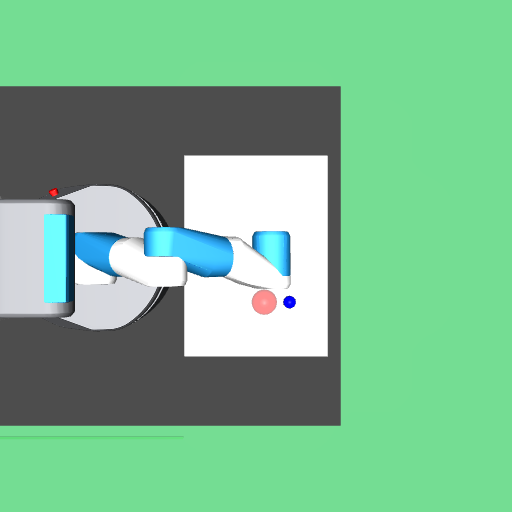}}
    \hspace{1em}
    \subfloat[Permissive safety budget\label{fig:reacher_vf_high}]{%
        \includegraphics[width=0.4\linewidth]{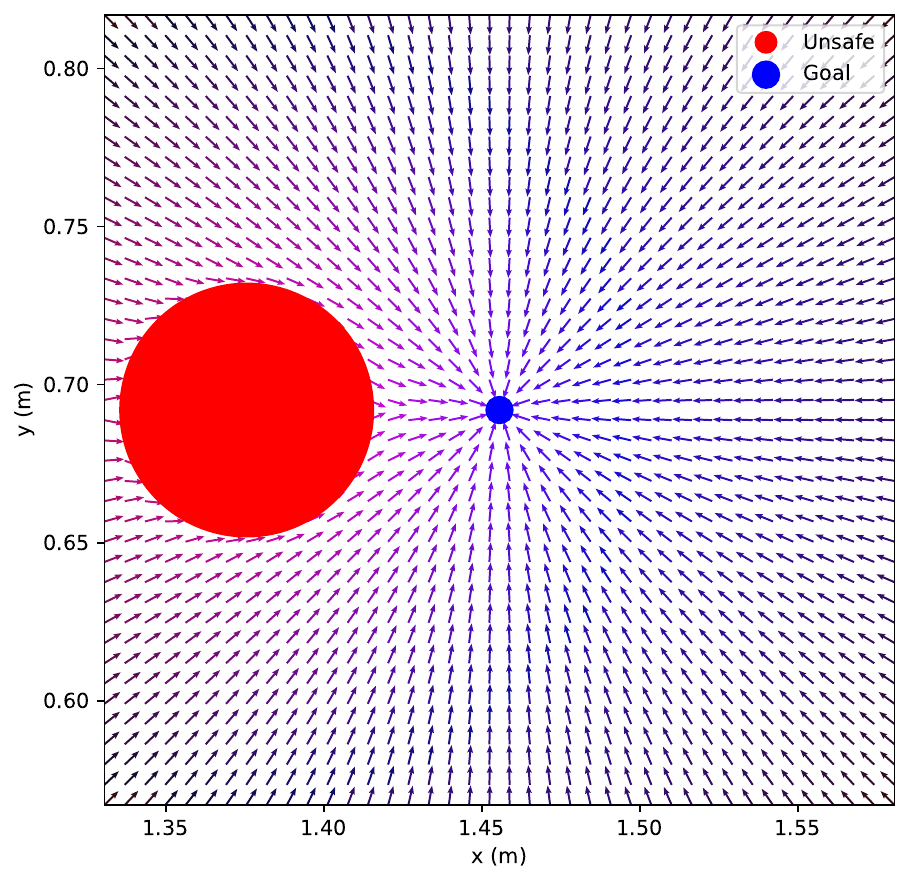}}

    \vspace{2em}

    \subfloat[Balanced safety budget\label{fig:reacher_vf_medium}]{%
        \includegraphics[width=0.4\linewidth]{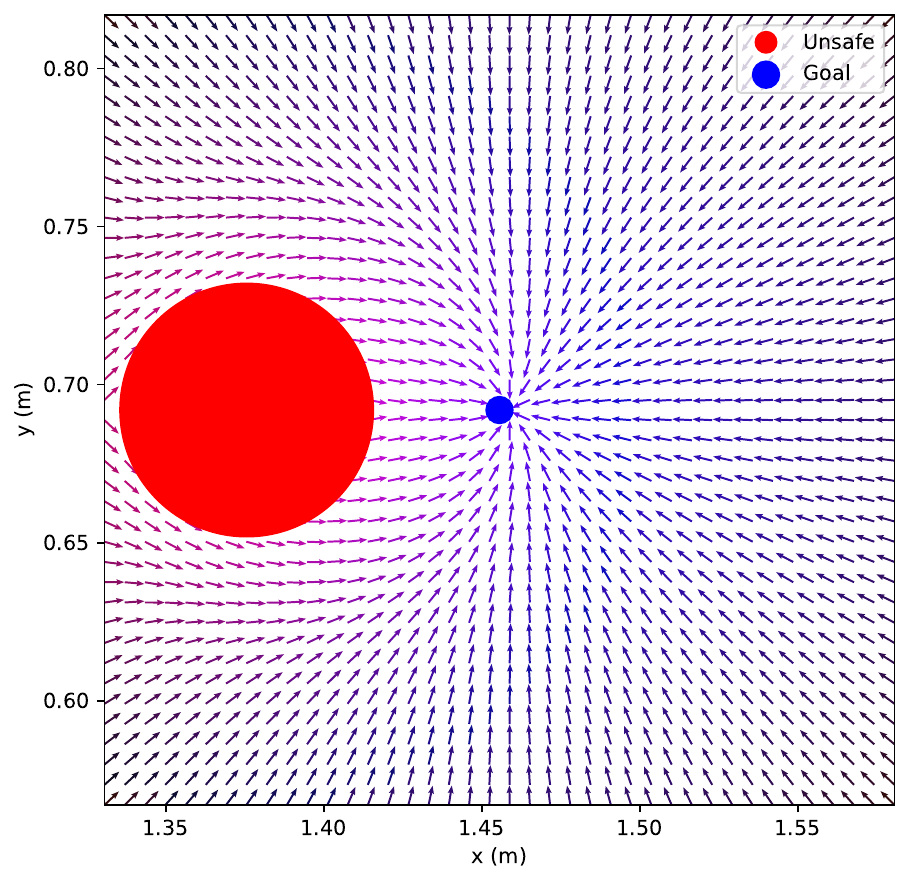}}
    \hspace{1em}
    \subfloat[Restrictive safety budget\label{fig:reacher_vf_low}]{%
        \includegraphics[width=0.4\linewidth]{plots/3.pdf}}
    \caption{Vector fields of RCI policies in \texttt{Safe-FetchReach} across three safety budgets. Arrows depict policy actions and color encodes the critic's value estimates. (\ref{fig:reacher_vf_high}) Higher budgets permit direct goal approach. (\ref{fig:reacher_vf_medium}) Intermediate budgets balance hazard avoidance with goal-reaching. (\ref{fig:reacher_vf_low}) Restrictive budgets enforce strong hazard avoidance at the expense of task success.}
    \label{fig:reacher_vf}
\end{figure}

The vector fields confirm that the learned policies adapt coherently to the imposed constraint. With a permissive budget (Figure~\ref{fig:reacher_vf_high}), the policy pursues the target directly, tolerating hazard proximity as the constraint permits risk-taking behavior. At an intermediate budget (Figure~\ref{fig:reacher_vf_medium}), the policy steers around the hazard while maintaining progress toward the goal, demonstrating learned avoidance without complete task abandonment. Under a restrictive budget (Figure~\ref{fig:reacher_vf_low}), the policy exhibits strong repulsive behavior away from the hazard region, prioritizing constraint satisfaction even at the cost of task success. This smooth degradation from goal-pursuit to hazard-avoidance as the constraint tightens reflects the Lagrangian mechanism: as $d$ decreases, the optimal multiplier $\lambda^*$ increases, reweighting the objective toward cost minimization over reward maximization.

The value estimates confirm this interpretation, with lower values near the hazard region indicating that the critic has learned to associate proximity with elevated risk. That spatially coherent avoidance strategies emerge from trajectory-level stop-feedback alone---without access to dense cost annotations---validates that the redistributed costs carry genuine per-step safety semantics.

\section{Experimental Details}
\subsection{Benchmark Environments}
\label{app:env-details}

Both environments share a common annotation protocol: if a safety violation occurs during data collection, the trajectory is truncated at the first unsafe step and assigned an episodic cost label $C(\tau)=1$; otherwise $C(\tau)=0$. This stop-feedback structure provides sparse supervision that indicates \emph{that} a violation occurred, but not which earlier decisions contributed to it.

\subsubsection{\texttt{HighwayEnv}}

\textbf{Task.}
The agent controls an ego vehicle on a multilane highway populated with traffic, aiming to maintain forward progress while avoiding unsafe maneuvers.

\textbf{State space.}
Each state is a fixed-size array of kinematic features for the ego and $V{-}1$ nearby vehicles, including presence, relative position $(x,y)$, velocities $(v_x,v_y)$, and orientation $(\cos h, \sin h)$. Features are normalized and expressed in ego-centric coordinates, with zero-padding to maintain fixed dimensionality.

\textbf{Action space.}
Actions are two-dimensional continuous controls $a = [a_\text{throttle},\, \delta_\text{steer}]^\top \in [-1,1]^2$, mapped to acceleration $\dot{v} \in [-5,5]\;\text{m/s}^2$ and steering $\delta \in [-0.785,0.785]\;\text{rad}$.

\textbf{Reward function.}
We modify the native reward to emphasize forward velocity, simulating a safety-unaware ego vehicle:
\begin{equation}
  R(s,a) = \alpha \cdot \frac{v - v_{\min}}{v_{\max} - v_{\min}},
\end{equation}
with $\alpha > 0$. The collision coefficient is set to zero, so collisions do not incur an explicit reward penalty; however, collisions still terminate the episode via the simulator's built-in termination condition. This ensures that reward-optimized behavioral policies can remain aggressive, while safety is enforced exclusively through the stop-feedback cost signal.

\textbf{Unsafe criterion.}
A transition is labeled unsafe when the ego vehicle comes within distance $0.2$ of another vehicle:
\begin{equation}
  c^{\text{sparse}}(s,a) = \mathbf{1}\!\Big\{\min_{j \in \{1,\dots,V-1\}} \lVert p_\text{ego} - p_j \rVert_2 \leq 0.2 \Big\}.
\end{equation}

\begin{figure}[h]
    \centering
    \includegraphics[width=0.6\linewidth]{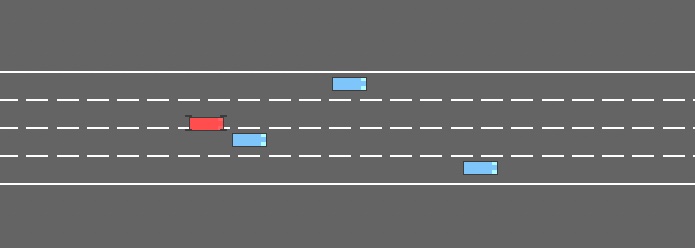}
    \caption{\texttt{HighwayEnv}: a transition is labeled unsafe when the ego vehicle comes within distance $0.2$ of another vehicle.}
    \label{fig:unsafe_highway}
\end{figure}

\subsubsection{\texttt{Safe-FetchReach}}

\textbf{Task.}
A 7-DOF Fetch robot must move its end-effector to a sampled Cartesian target. We introduce a spherical hazard region to assess safety.

\textbf{State space.}
Observations are tuples $(o, g_\text{ach}, g_\text{des})$, where $o \in \mathbb{R}^{10}$ encodes end-effector position, finger joint states, and velocities, and $g_\text{ach}, g_\text{des} \in \mathbb{R}^3$ denote achieved and desired goals.

\textbf{Action space.}
Actions are $a = [\Delta x, \Delta y, \Delta z, a_\text{grip}]^\top \in [-1,1]^4$, with the first three dimensions mapped to Cartesian displacements and the fourth controlling the gripper (unused in this task).

\textbf{Reward function.}
The dense shaping reward is:
\begin{equation}
  R(s,a) = -\lVert g_\text{ach} - g_\text{des} \rVert_2.
\end{equation}

\textbf{Unsafe criterion.}
We define a spherical hazard region $\mathcal{H} = \{p \in \mathbb{R}^3 : \lVert p - h \rVert_2 \leq r\}$ with center $h$ and radius $r$. A transition is labeled unsafe at the first step where the end-effector enters $\mathcal{H}$:
\begin{equation}
  c^{\text{sparse}}(s,a) = \mathbf{1}\{g_\text{ach} \in \mathcal{H}\}.
\end{equation}

\begin{figure}[h]
    \centering
    \includegraphics[width=0.3\linewidth]{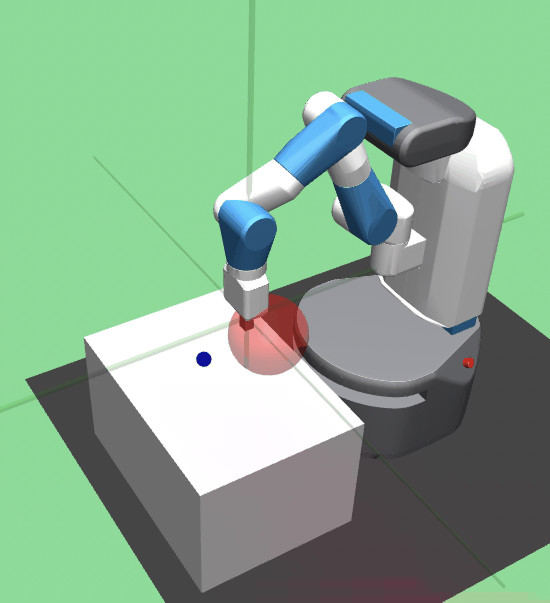}
    \caption{\texttt{Safe-FetchReach}: a transition is labeled unsafe when the end-effector enters the spherical hazard region $\mathcal{H}$.}
    \label{fig:unsafe_reacher}
\end{figure}

\end{document}